\documentclass[a4paper,numbers]{article}
\usepackage{amsmath,amssymb, amsthm, amsfonts,tikz,enumerate,mathtools}
\usepackage{bbold}
\usepackage{comment}
\usepackage[ruled,vlined]{algorithm2e}
\usepackage{graphicx}
\graphicspath{ {./images/} }
\usepackage{orcidlink}
\providecommand{\keywords}[1]
{
  \small	
  \textbf{\textit{Keywords---}} #1
}

\begin{document}
\title{Three-way Decisions with Evaluative Linguistic Expressions}
\author{Stefania Boffa \orcidlink{0000-0002-4171-3459}  \and
Davide Ciucci \orcidlink{0000000280837809}
}

\newtheorem{definition}{Definition}
\newtheorem{corollary}{Corollary}
\newtheorem{remark}{Remark}
\newtheorem{example}{Example}
\newtheorem{proposition}{Proposition}
\newtheorem{theorem}{Theorem}

\maketitle
\begin{abstract}
    We propose a linguistic interpretation of three-way decisions, where  the regions of acceptance, rejection, and non-commitment are constructed by using the so-called evaluative linguistic expressions, which  are expressions of natural language such as small, medium, very short, quite roughly strong, extremely good, etc. Our results highlight new connections between two different research areas: three-way decisions and
the theory of evaluative linguistic expressions.
\end{abstract}
\section{Introduction}
The theory of three-way decisions (TWD) divides a finite and non-empty universe into three disjoint sets, which are called positive, negative, and boundary regions. These regions respectively induce positive, negative, and boundary rules: a positive rule makes a decision of acceptance, a negative rule makes a decision of rejection, and a boundary rule makes an abstained or non-committed decision \cite{yao2009three,yao2010three}. The concept of three-way decisions was originally introduced in Rough Set Theory  
\cite{yao2009three,pawlak1982rough}   
and until today, it has been widely studied and applied to many decision-making problems (see \cite{pauker1980threshold,4028275,lin2002granular,skowron2006conflict} for some examples). Thus, several approaches have been proposed to generate the three regions; one of them is based on probabilistic rough sets, which generalizes probabilistic rough sets \cite{pawlak1988rough,yao2008probabilistic} where the three regions are constructed using a pair of thresholds and the notion of conditional probability (in this case, the regions are called probabilistic positive, negative, and boundary regions). 

The contribution of this article is to provide a linguistic interpretation of the positive, negative, and boundary regions. So, we propose a three-way decision method based on the concept of \emph{evaluative linguistic expressions},  which  are expressions of natural language such as \emph{small, medium, very short, quite roughly strong, extremely good, etc}. These are already considered in the majority of applications of
fuzzy modelling. Since we use evaluative linguistic expressions to evaluate the size of sets, we focus on the expressions involving the adjectives \emph{small}, \emph{medium}, and \emph{big} that can be preceded by an adverb; examples are \emph{very small}, \emph{roughly medium}, and \emph{extremely big}.  Mathematically, an evaluative linguistic expression is modelled by a function $Ev: [0,1]\to [0,1]$.
The formal theory of evaluative linguistic expressions is introduced and explained in \cite{novak2008comprehensive,novak2015evaluative,novak2015fuzzy,novak2016insight}.


The positive, negative, and boundary regions of a non-empty and finite universe $U$ are defined here starting from a subset $X$ of $U$, an equivalence relation $\mathcal{R}$ on $U$ (i.e. $\mathcal{R}$ is reflexive, symmetric and transitive), an evaluative linguistic expression $Ev$, and a pair of thresholds $(\alpha, \beta)$ with $0 \leq  \beta < \alpha \leq 1$. Then, an object $x$ belongs to the positive region when \emph{the size  of  $[x]_{\mathcal{R}} \cap X$ evaluated w.r.t. $Ev$ is at least $\alpha$}, where $[x]_{\mathcal{R}}$ is the equivalence class of $x$ w.r.t. $\mathcal{R}$. Analogously,  $x$ belongs to the negative region when \emph{the size  of  $[x]_{\mathcal{R}} \cap X$ evaluated w.r.t. $Ev$ is at most $\beta$}. Finally, the remaining elements form the boundary region. In order to obtain the three regions, the size of $X \cap [x]_{\mathcal{R}}$ is quantified using a fuzzy measure \cite{choquet1954theory,sugeno1974theory}.

The role of evaluative linguistic expressions in the context of three-way decisions can be better understood with the following example. 
\begin{example}
    Suppose that the number of buses between the University of Buenos Aires and the rest of the city has to be increased from 7 am to 8 am. Thus, we intend to understand which city areas need buses the most, as resources are limited. Let us denote the areas of the city with $A_1, \ldots, A_n$ and map each area $A_i$ with the set $S_{A_i}$ made of all students of the university who live in $A_i$. Thus, $S_{A_1}, \ldots, S_{A_n}$ can be seen as the equivalence classes w.r.t. the relation $\mathcal{R}$ on the set of all students of the University of Buenos Aires living in the city:  $x \mathcal{R} y$ if and only if $x$ and $y$ live in the same area. 
Based on a survey, we also consider a set $X$ made of all students that usually take a bus to the university in the slot time [7 am, 8 am]. We also choose $(\alpha,\beta)=(0.3,0.6)$ and $Ev=extemely \ big$. 
We construct three regions in the following way. The positive region is the union of $S_{A'_1}, \ldots, S_{A'_k}$ (with $\{A'_1, \ldots, A'_k\} \subseteq \{A_1, \ldots, A_n\}$) so that the amount of students of $S_{A_i}$ that take a bus from 7 am to 8 am is ``extremely big'' with a value of at least 0.6. Similarly,  the negative region is the union of $S_{A^*_1}, \ldots, S_{A^*_h}$ (with $\{A^*_1, \ldots, A^*_h\} \subseteq \{A_1, \ldots, A_n\}$) so that the amount of students of $S_{A^*_i}$ that take a bus from 7 am to 8 am is extremely big with a value of at most 0.3. All other students form the boundary region. The final decision is immediate: the buses are certainly increased for the areas $A'_1, \ldots, A'_k$, but not for $A^*_1, \ldots, A^*_h$. Furthermore, the decision is postponed for the remaining areas (thati is, for each $A_i \notin \{A'_1, \ldots, A'_k\} \cup \{A^*_1, \ldots, A^*_h\}$). In order to make a decision in those areas,  for example, we could take into account the workers (besides the students)  that need a bus in the slot of time [7 am, 8 am].
\end{example}

The choice of $Ev$ depends on the context where the three regions are used.
Indeed, in the previous example, we have chosen \emph{extremely big} in order to select the areas where a large number of students catch the bus from 7 am to 8 am. However, if we focus on the inverse problem (namely we need to eliminate some existing bus rides), then we should identify the areas where there are fewer students taking the bus in the time slot [7 am, 8 am]. Therefore, in this case, the evaluative linguistic expression \emph{extremely small} is more appropriate to construct the three regions.

A significant contribution of this article is  providing a linguistic and novel interpretation of the positive, negative, and boundary regions already determined with probabilistic rough sets. 
Consequently, the reasons for decisions of acceptance, rejection, and non-commitment can be explained in terms of expressions of natural language. Of course, the advantage is that non-technical users dealing with TWD models can better understand the  reliability of the procedures related to the final decisions. This is in line with the scope of \emph{Explainable Artificial Intelligence (XAI)}, which is a new
approach to AI emphasizing the ability of machines to give sound motivations about their decisions and behaviour \cite{moral2021explainable}.

The article is organized as follows. The next section reviews
some basic notions  regarding probabilistic three-way decisions and the concept of evaluative linguistic expressions. Also, the notion of fuzzy measure is recalled. Section \ref{sec:3} presents a new model of three-way decisions based on the theory of evaluative linguistic expressions. As a consequence, a linguistic generalization of Pawlak rough sets is introduced.  Finally, Section \ref{sec:4} connects the TWD models based on evaluative linguistic expressions and  probabilistic rough sets. 
In particular, confining to the evaluative linguistic expressions modelled by increasing functions, we find the class of thresholds so that the corresponding probabilistic positive, negative, and boundary regions are equal to those generated by a given evaluative linguistic expression.







\section{Preliminaries}

In the following, we consider a finite universe $U$, a subset $X$ of $U$, and an equivalence relation $\mathcal{R}$ on $U$ (i.e. $\mathcal{R}$ is reflexive, symmetric, and transitive). Moreover, we indicate the equivalence class of $x \in U$ w.r.t. $\mathcal{R}$ with $[x]_{\mathcal{R}}$. 

\subsection{Three-way decisions with probabilistic rough sets}
This subsection recalls the fundamental notions of three-way decisions based on probabilistic rough sets.


Viewing $X$ and $[x]_{\mathcal{R}}$ as events of $U$, the symbol $Pr(X|[x]_{\mathcal{R}})$ denotes the \emph{conditional probability} of $X$ given $[x]_{\mathcal{R}}$, i.e.
	\begin{equation} \label{eq:pro}
	    Pr(X|[x]_{\mathcal{R}})=\frac{|[x]_\mathcal{R} \cap X|}{|[x]_\mathcal{R}|}.
	\end{equation}
Then, three special subsets of $U$  are determined by using \eqref{eq:pro} and a pair of thresholds as shown by the next definition. 
\begin{definition} \label{def:PRS}
Let $\alpha, \beta \in [0,1]$ such that $\beta < \alpha$, the $(\alpha,\beta)$-\emph{probabilistic positive, negative} and \emph{boundary regions}  are respectively the following:
\begin{enumerate}
\item[(i)] $POS_{(\alpha,\beta)}(X)=\{x \in U \ | \ Pr(X | [x]_{\mathcal{R}}) \geq \alpha\}$, 
\item[(ii)] $NEG_{(\alpha,\beta)}(X)=\{x \in U \ | \ Pr(X|[x]_{\mathcal{R}}) \leq \beta\}$, 
\item[(iii)] $BND_{(\alpha,\beta)}(X)=\{x \in U \ | \ \beta < Pr(X|[x]_{\mathcal{R}}) < \alpha\}$.
\end{enumerate}
\end{definition}
We put
\begin{equation} 
\mathcal{T}_{(\alpha,\beta)}(X)=\{POS_{(\alpha, \beta)}(X), NEG_{(\alpha, \beta)}(X), BND_{(\alpha, \beta)}(X)\}
\end{equation}
and we say that $\mathcal{T}_{(\alpha,\beta)}(X)$ is a \emph{tri-partition} of $U$ due to the following remark \footnote{By a tri-partition, we mean a partition of $U$ made of three equivalence classes. On the other hand,  $\{POS_{(\alpha, \beta)}(X), NEG_{(\alpha, \beta)}(X), BND_{(\alpha, \beta)}(X)\}$ could collapse into a bi-partition or the whole universe when one or two of its sets are empty.}. 
\begin{remark}
The three regions of of $\mathcal{T}_{(\alpha,\beta)}(X)$ are mutually disjoint, i.e. $A \cap B =\emptyset$ for each $A, B \in \{POS_{(\alpha,\beta)}(X), NEG_{(\alpha,\beta)}(X), BND_{(\alpha,\beta)}(X)\}$ with $A \neq B$, and they cover the universe $U$, i.e.  
\begin{equation} \label{eq:tr1}
POS_{(\alpha,\beta)}(X) \cup NEG_{(\alpha,\beta)}(X) \cup BND_{(\alpha,\beta)}(X)=U.
\end{equation}
\end{remark}

In the context of three-way decisions,  the following rules are considered: let $x \in U$,
\begin{itemize}
\item if $x \in POS_{(\alpha,\beta)}(X)$, then $x$ is accepted;
\item if $x \in NEG_{(\alpha,\beta)}(X)$, then $x$ is rejected;
\item if $x \in BND_{(\alpha,\beta)}(X)$, then we abstain on $x$.
\end{itemize}
The values $Pr(X | [x]_\mathcal{R})$ represents the \emph{accuracy} or
\emph{confidence} of the rules:
\begin{itemize}
\item the higher  $Pr(X | [x]_\mathcal{R})$ is, the more confident we are that  $x \in POS_{(\alpha,\beta)}(X)$ is correctly accepted, 
\item the lower $Pr(X | [x]_\mathcal{R})$ is, the more confident we are that $x \in NEG_{(\alpha,\beta)}(X)$ is correctly rejected.  
\end{itemize}

Definition \ref{def:PRS} is strictly related to the notion of probabilistic rough sets. 
\begin{definition}
The \emph{$(\alpha,\beta)$-probabilistic rough set} of $X$ is the pair $$(\mathcal{L}_{(\alpha,\beta)}(X),\mathcal{U}_{(\alpha,\beta)}(X)),$$ where 
\begin{center}
$\mathcal{L}_{(\alpha,\beta)}(X)=POS_{(\alpha, \beta)}(X)$ \ \ and \ \  $\mathcal{U}_{(\alpha,\beta)}(X)=POS_{(\alpha, \beta)}(X) \cup BND_{(\alpha, \beta)}(X)$, 
\end{center}
which are respectively called $(\alpha,\beta)-$ \emph{lower and upper approximations} of $X$.
\end{definition}

\begin{remark} \label{rem:1}
When $\alpha=1$ and $\beta=0$, $(\mathcal{L}_{(\alpha,\beta)}(X),\mathcal{U}_{(\alpha,\beta)}(X))$ is the rough set $(\mathcal{L}(X),\mathcal{U}(X))$ of $X$ defined by Pawlak in \cite{pawlak1982rough}, namely 
\begin{equation} \label{eq:RS}
(\mathcal{L}(X),\mathcal{U}(X))=(\{ x \in U \ | \ [x]_R \subseteq X\},\{x \in U \ | \ [x]_R \cap X \neq \emptyset\}).
\end{equation}
The sets $\mathcal{L}(X)$ and $\mathcal{U}(X)$ are respectively called \emph{lower} and \emph{upper approximations} of $X$ w.r.t. $\mathcal{R}$. 
\end{remark}

\subsection{Evaluative Linguistic Expressions} \label{sub:expressions}

This subsection reviews concepts and results that are found in \cite{novak2008comprehensive,czbook} and it recalls the notion of fuzzy measure.  

\emph{Evaluative linguistic expressions} are special expressions of natural language, which people commonly employ to evaluate, judge, estimate, and in many other situations. Examples of evaluative linguistic expressions are \emph{small, medium, big, about twenty-five, roughly one hundred, very short, more or less deep, not very tall, roughly warm or medium-hot}, etc. For convenience, we will often omit the adjective “linguistic” and use only the term “evaluative expressions”. 
The simplest evaluative expressions are called \emph{pure evaluative expressions} and have the following structure:
\begin{center}	
 \emph{$\langle \mbox{linguistic hedge} \rangle \langle \mbox{TE-adjective}\rangle,$}
 \end{center}
where 
\begin{itemize}
\item a linguistic hedge is an adverbial modification such as \emph{very, roughly, approximately, significantly} and 
\item a TE-adjective is an adjective such us \emph{good, medium}, \emph{big}, \emph{short}, etc. TE stands for \emph{trichotomous evaluative}, indeed TE-adjectives typically form pairs of antonyms like \emph{small} and  \emph{big} completed by a middle member, which is \emph{medium} in the case of \emph{small} and \emph{big}. Other examples are “\emph{weak, medium-strong,} and 
\mbox{strong}" and “\emph{soft, medium-hard,} and \emph{hard}". 
\end{itemize}
The \emph{empty linguistic hedge} is employed to deal with evaluative expressions made of only a TE-adjective; hence,  \emph{small}, \emph{medium}, and \emph{big} are considered evaluative expressions. Other pure evaluative expressions are the fuzzy numbers like \emph{about twenty-five}. Two or more pure evaluative expressions can be connected to form \emph{negative evaluative expressions} like \emph{“NOT very small"} and \emph{compound evaluative expressions} like \emph{“very expensive AND extremely small"} and \emph{“very expensive OR extremely small"}. 

The semantics of evaluative expressions is based on the essential concepts of \emph{context}, 
\emph{intension}, and \emph{extension}. 
\begin{itemize}
\item The \emph{context} is a state of the world at a given time  and place in which an evaluative expression appears. Each context is represented by a linearly ordered scale, which is bounded by $s$ and $b$. Moreover, a context is given by a triple $\omega=\langle s,m,b \rangle$, where $s$ is the “most typical” small value, $m$ is the “most typical” medium value, and $b$ is the “most typical” big value.   
For example, suppose that evaluative expressions are used to evaluate the size of apartments. If we are thinking of apartments for one person, then we could choose 
 $\omega_1=\langle 40, 70, 100 \rangle$ as context, which means that flats measuring $40 \  m^2$, $70 \ m^2$, and $100 \ m^2$ are typically small, medium and big, respectively. On the other hand, when changing context and thinking of apartments for a family of 5 people, the  context $\omega_5=\langle 70, 120, 160 \rangle$ is more appropriate. 
%
\item The \emph{intension} of an evaluative expression 
is a function mapping each context into a fuzzy set of a given universe. Taking up  the previous example, we consider a universe made of four apartments $\{a_1,a_2,a_3,a_4\}$, then the intension of $small$ is the map $Int_{small}$ that assigns to the context $\omega_5$ the fuzzy set $A_{\omega_5}$ so that $A_{\omega_5}(a_i)$ is the degree to which \emph{$a_i$ is small in the context $\omega_5$} (namely, \emph{$a_i$ is small for 5 people}).  
\item The \emph{extension} of an evaluative expression $Ev$ is a fuzzy set determined by the intension of $Ev$, given a context $\omega$. Concerning the previous example, $Int_{small}(\omega_5)=A_{\omega_5}$ is an example of an extension of \emph{small}.
\end{itemize}


In this article, we confine to the TE adjectives \emph{small}, \emph{medium}, and \emph{big} because we use evaluative expressions to evaluate the size of sets.  So, let $X$ be a subset of a universe $U$, we will say that the size of $X$ w.r.t. to the size of $U$ is \emph{very small}, \emph{extremely big}, etc. Furthermore, we confine to the \emph{standard context}, which is $\langle 0,0.5,1 \rangle$. Finally, since sizes are expressed by means of a fuzzy measure (by Example \ref{ex:fuzzymeasure},  the measure of the size of  a set $X$ is a  value of $[0,1]$), the extensions of our evaluative expressions are functions from $[0,1]$ to $[0,1]$, which have a specific formula.  The extension of an evaluative expression like $\langle linguist \ hedge \rangle \langle TE-adjective\rangle$ with $TE-adjective \in \{small, medium, big\}$ is obtained by composing two functions, one models the linguistic hedge and the other models the TE-adjective. The function describing a linguistic hedge depends on three parameters, which are experimentally estimated (see \cite{czbook} for more details). 

In what follows, we provide the formula of $\neg Sm:[0,1] \to [0,1]$, $BiVe:[0,1] \to [0,1]$, and $BiEx:[0,1] \to [0,1]$, which are the extensions of the evaluative expressions \emph{not small}, \emph{very big}, and \emph{extremely big}, where the context $\langle 0, 0.5, 1 \rangle$ is fixed \footnote{We have got the formulas of $BiVe$ and $BiEx$ using the function $\nu_{a,b,c}(LH(\omega^{-1}))$ and Table 5.1 given in \cite{czbook}. 
Concerning the formula of $\neg Sm$, we have considered that $\neg Sm(x)=1-Sm(x)$. After that, we have found the formula of $Sm$ using the function $\nu_{a,b,c}(RH(\omega^{-1}))$ and Table 5.1 of \cite{czbook}.}. 

\begin{equation} \label{eq:nSm}
\neg Sm(x)=
\begin{dcases}
1 & \mbox{ if }  x \in [0.275,1],\\ 
1-\frac{(0.275-x)^2}{0.02305} & \mbox{ if }  x \in (0.16,0.275) \\ \frac{(x-0.0745)^2}{0.01714} & \mbox{ if }  x \in (0.0745, 0.16]  \\
0 & \mbox{ if }  x \in [0, 0.0745]
\end{dcases}
\end{equation}
\begin{equation} \label{eq:vBig}
BiVe(x)=
\begin{dcases}
1 & \mbox{ if }  x \in [0.9575, 1],\\ 
1-\frac{(0.9575-x)^2}{0.00796} & \mbox{ if }  x \in [0.895,0.9575), \\ \frac{(x-0.83)^2}{0.00828} & \mbox{ if }  x \in (0.83, 0.895),  \\
0 & \mbox{ if }  x \in [0, 0.83].
\end{dcases}
\end{equation}
\begin{equation} \label{eq:eBig}
BiEx(x)=
\begin{dcases}
1 & \mbox{ if }  x \in [0.995, 1],\\ 
1-\frac{(0.995-x)^2}{0.00495} & \mbox{ if }  x \in [0.95,0.995), \\ \frac{(x-0.885)^2}{0.00715} & \mbox{ if }  x \in (0.885, 0.95),  \\
0 & \mbox{ if }  x \in [0, 0.885].
\end{dcases}
\end{equation}
\begin{remark} \label{rem:new} The evaluative expressions $\neg Sm$, $BiVe$, and $BiEx$ have a special role: they are respectively used to construct the formula of fuzzy quantifiers \emph{many}, \emph{most}, and  \emph{almost all} \cite{novak2008formal}.
\end{remark}

A further class of linguistic expressions is $\{\Delta_t: [0,1] \to \{0,1\}\ | \ t \in [0,1]\}$, where given $t \in [0,1]$ the formula of $\Delta_t$ is the following: let $a \in [0,1]$,
\begin{equation} \label{eq:delta}    \Delta_t(a)=\begin{cases} 1 & \mbox{ if } a \geq t\\ 0 & \mbox{ otherwise. } \end{cases}
\end{equation}

In the sequel, we need  the notion of fuzzy measure \cite{choquet1954theory,sugeno1974theory}. 
\begin{definition}
Let $U$ be a finite universe,  a mapping $\varphi:2^U \to \mathbb{R}$ is called \emph{fuzzy measure} if and only if
\begin{enumerate} [(a)]
\item $\varphi(\emptyset)=0$;
\item if $X \subseteq Y$ then $\varphi(X) \leq \varphi(Y)$, for each $X,Y \subseteq U$ (monotonicity). 
\end{enumerate}
\end{definition}
A fuzzy measure $\varphi$  is called \emph{normalized} or \emph{regular} if $\varphi(U)=1$. 

In this paper, we focus on the  normalized fuzzy measure given by the next example.
\begin{example} \label{ex:fuzzymeasure}
Let $U$ be a finite universe, the function $f:2^U \to \mathbb{R}$ that assigns  
$\dfrac{|Y|}{|U|}$ to each $Y \subseteq U$ is a fuzzy measure.
 
 The value $\dfrac{|Y|}{|U|}$ belongs to [0,1] and measures  \emph{``how much $Y$ is large with respect to $U$ in the scale $[0,1]$"}.
\end{example}

Let us observe that in Probability theory  $\dfrac{|Y|}{|U|}$ represents  \emph{``how likely the event $Y$ is to occur"}. 

\section{Three-way decisions with linguistic expressions} \label{sec:3}
This subsection proposes a novel model for three-way decisions, which is based on the concept of evaluative linguistic expressions previously described.

In the sequel, we use the symbol $\mathcal{E}$ to denote the collection of the extensions of all evaluative expressions in the context $\langle 0, 0.5, 1 \rangle$. 


Therefore, let $Ev \in \mathcal{E}$, let $X \subseteq U$, and let $\alpha,\beta \in [0,1]$ with $\beta < \alpha$,  three  regions of $U$ are determined. 
In particular, the region of a given element $x \in U$ is found by taking into account the following steps:
\begin{enumerate}
\item computing $Ev\left(\dfrac{|X \cap [x]_\mathcal{R}|}{|[x]|_\mathcal{R}}\right)$, which is the evaluation of \emph{the size of $X \cap [x]_\mathcal{R}$ w.r.t. the size of $[x]_\mathcal{R}$} by using $Ev$;
\item comparing  $Ev\left(\dfrac{|X \cap [x]_\mathcal{R}|}{|[x]|_\mathcal{R}}\right)$ with the thresholds $\alpha$ and $\beta$.
\end{enumerate}
For example, regarding point 1, if \emph{Ev} models the evaluative expression 
``\emph{significantly big}", 
then $Ev\left(\dfrac{|X \cap [x]_\mathcal{R}|}{|[x]_\mathcal{R}|}\right)$ measures 
\begin{center}
\emph{``how much the size of $X \cap [x]_\mathcal{R}$ is significantly big w.r.t. the size of $[x]_\mathcal{R}$".} 
\end{center}
Equivalently, we are saying that 
\begin{center}
\emph{``the size of the set of the elements of $[x]_\mathcal{R}$ that also belong to $X$ is significantly big with the truth degree $Ev\left(\dfrac{|X \cap [x]_\mathcal{R}|}{|[x]_\mathcal{R}|}\right)$".}  
\end{center}
 
Observe that $\dfrac{|X \cap [x]_{\mathcal{R}}|}{|[x]_{\mathcal{R}}|}$ syntactically coincides with the conditional probability (see \eqref{eq:pro}), but here it has a different interpretation: it is the fuzzy measure specified by Example \ref{ex:fuzzymeasure}.

Formally, the three regions of $U$ determined by an evaluative expression are given by the following definition.

\begin{definition} \label{def:ee}
Let $Ev \in \mathcal{E}$, 
 the $(\alpha,\beta)$-linguistic positive, negative, and boundary regions induced by $Ev$ are respectively the following:
\begin{enumerate}
    \item [(i)]  $POS^{Ev}_{(\alpha,\beta)}(X)=\left\{x \in U \ | \ Ev\left(\dfrac{|[x]_\mathcal{R} \cap X|}{|[x]_\mathcal{R}|}\right) \geq \alpha\right\}$;
    \item [(ii)] $NEG^{Ev}_{(\alpha,\beta)}(X)=\left\{x \in U \ | \ Ev\left(\dfrac{|[x]_\mathcal{R} \cap X|}{|[x]_\mathcal{R}|}\right) \leq \beta\right\}$;
     \item [(iii)] $BND^{Ev}_{(\alpha,\beta)}(X)=\left\{x \in U \ | \ \beta < Ev\left(\dfrac{|[x]_\mathcal{R} \cap X|}{|[x]_\mathcal{R}|}\right) < \alpha\right\}$.
\end{enumerate}
\end{definition}
We put 
\begin{equation} \label{lab:tr1}
\mathcal{T}_{(\alpha,\beta)}^{Ev}(X)=\{POS_{(\alpha,\beta)}^{Ev}(X), NEG_{(\alpha,\beta)}^{Ev}(X), BND_{(\alpha,\beta)}^{Ev}(X)\}
\end{equation}
and we say that $\mathcal{T}_{(\alpha,\beta)}^{Ev}(X)$ is a tri-partition of $U$.
\begin{remark} \label{rem:tri}
 The three regions of $\mathcal{T}_{(\alpha,\beta)}^{Ev}(X)$ are mutually disjoint, i.e. $A \cap B =\emptyset$ for each $A, B \in \{POS_{(\alpha,\beta)}^{Ev}(X), NEG_{(\alpha,\beta)}^{Ev}(X), BND_{(\alpha,\beta)}^{Ev}(X)\}$ with $A \neq B$, and they cover the universe $U$, i.e. 
\begin{equation} \label{eq:tr}
POS_{(\alpha,\beta)}^{Ev}(X) \cup NEG_{(\alpha,\beta)}^{Ev}(X) \cup BND_{(\alpha,\beta)}^{Ev}(X)=U.
\end{equation}

\end{remark}

\begin{remark}
Let us focus on the evaluative expressions \emph{not small}, \emph{very big}, \emph{extremely big}, and \emph{utmost}. The first three expressions are respectively modelled by \eqref{eq:nSm}, \eqref{eq:vBig}, and \eqref{eq:eBig}. 
 According to Remark \ref{rem:new}, $\neg Sm$, $Bi Ve$, and $Bi Ex$ respectively appear into the formula of quantifiers \emph{many}, \emph{most}, and \emph{almost all}. Moreover, as explained in \cite{boffa2021proposal} (see Lemma 4.5), considering that $X$ and $[x]_{\mathcal{R}}$ are crisp set,  $\neg Sm\left(\dfrac{|X \cap [x]_\mathcal{R}|}{|[x]_\mathcal{R}|}\right)$, $BiVe\left(\dfrac{|X \cap [x]_\mathcal{R}|}{|[x]_\mathcal{R}|}\right)$, and $BiEx\left(\dfrac{|X \cap [x]_\mathcal{R}|}{|[x]_\mathcal{R}|}\right)$ exactly coincide with the formula of quantifiers \emph{many}, \emph{most}, and \emph{almost all}. Hence,  
 they have the following meaning:
\begin{itemize}
    \item $\neg Sm\left(\dfrac{|X \cap [x]_\mathcal{R}|}{|[x]_\mathcal{R}|}\right)$ is degree to which ``\emph{many} objects of $[x]_\mathcal{R}$ are in $X$",
     \item $BiVe\left(\dfrac{|X \cap [x]_\mathcal{R}|}{|[x]_\mathcal{R}|}\right)$ is degree to which ``\emph{most} objects of $[x]_\mathcal{R}$ are in $X$",
      \item $BiEx\left(\dfrac{|X \cap [x]_\mathcal{R}|}{|[x]_\mathcal{R}|}\right)$ is degree to which ``\emph{almost all} objects of $[x]_\mathcal{R}$ are in $X$".
\end{itemize}
Moreover, the function $\Delta_1$ that is obtained by \eqref{eq:delta} and putting $t=1$, models the evaluative expression \emph{utmost} and corresponds to the quantifier \emph{all}. Therefore, $\Delta^1\left(\dfrac{|X \cap [x]_\mathcal{R}|}{|[x]_\mathcal{R}|}\right)$ is understood as the degree to which ``\emph{all} objects of $[x]_{\mathcal{R}}$ are in $X$".

Observe that here the universe of quantification coincides with  $[x]_{\mathcal{R}}$, which is always non-empty, considering that $\mathcal{R}$ is reflexive, hence $\{x\} \subseteq [x]_{\mathcal{R}}$. In mathematical logic, the assumption expressing that the universe of quantification must be non-empty 
is called \emph{existential import} (or \emph{presupposition}) 
\cite{chatti2013cube}. 
\end{remark}
\begin{remark}
Consider the evaluative expressions represented by \eqref{eq:delta}. Let $t \in [0,1]$, then $\Delta_t\left(\dfrac{|X \cap [x]_{\mathcal{R}}|}{|[x]_{\mathcal{R}}|}\right)$ is the degree to which “\emph{the size of the set of elements of $[x]_{\mathcal{R}}$ that also belong to $X$ is at least as
large as $t$ (in the scale [0,1])}". 
\end{remark}
\subsection{An illustrative example}
 \label{esempio}
In this subsection, we provide an example of how to use linguistic three-way decision to provide recommendations based on users's profile.

We consider a universe $U=\{u_1, \ldots, u_{32}\}$ made of users of online communities and the following equivalence relation $\mathcal{R}$ on $U$: let $x,y \in U$, $x \mathcal{R} y$ if and only if $x$ and $y$ belong to the same community. Therefore, $\mathcal{R}$ corresponds to the partition $\{C_1,\ldots,C_6\}$ of $U$, where $C_i$ is the set of users of $U$ belonging to the community $i$. In particular, we suppose that
\begin{multline*}
C_1=\{u_1, \ldots, u_5\}, \ C_2 = \{u_6, \ldots, u_{10}\}, \ C_3=\{u_{11},\ldots,u_{15}\}, \\
C_4=\{u_{16}, \ldots, u_{20}\}, \
C_5=\{u_{21},\ldots,u_{25}\}, \mbox{ and }  C_6=\{u_{26},\ldots,u_{32}\}.  
\end{multline*}

We use the symbol $X_T$ to denote the set of users of $U$ interested in a specific topic $T$.

For example, $$X_{Sport}=\{u_{10},u_{11},u_{12},u_{18}, u_{19}, u_{20}, u_{21}, u_{22}, u_{23},  u_{24},u_{26}\}$$
is the set of all users of $U$ interested in the topic \emph{Sport}.

Using three-way decisions based on evaluative expressions, we intend to select the most appropriate communities among $C_1, \ldots, C_6$ to which propose news related to the topic $T$.

If we choose $(\alpha,\beta)=(0.8,0.2)$ and the evaluative expression $\neg Sm$ corresponding to the fuzzy quantifier \emph{many}, then we decide to assign the news about the topic $T$ to the communities of $POS_{(0.8,0.2)}^{\neg Sm} (X_T)$. Indeed, enough users of $POS_{(0.8,0.2)}^{\neg Sm}(X_T)$ are interested in $T$:
$x \in POS_{(0.8,0.2)}^{\neg Sm}$ if and only if the degree to which \begin{center}  \emph{``many users of the community of $x$ are interested in $T$"}\end{center}
  is greater than or equal to 0.8. 

In the sequel, we determine the communities to which provide the news about $Sport$. To do this, we firstly compute the value $\dfrac{|X_{Sport} \cap [x]_\mathcal{R}|}{|[x]_\mathcal{R}|}$ for each $x \in U$:
\begin{equation} \label{eq:ex}
\dfrac{|X_{Sport} \cap [x]_\mathcal{R}|}{|[x]_\mathcal{R}|}=\begin{cases}
0 & \mbox{ if } x \in C_1,\\
0.14 & \mbox{ if } x \in C_6, \\
0.2 & \mbox{ if } x \in C_2, \\
0.4 & \mbox{ if } x \in C_3, \\ 
0.6 & \mbox{ if } x \in C_4, \\
0.8 & \mbox{ if } x \in C_5.
\end{cases}
\end{equation}

According to the definition of $\neg Sm$ that is given by \eqref{eq:nSm}, we get
$\neg Sm(0)=0$, $\neg Sm (0.14)=0.25$, $\neg Sm(0.2)=0.75$ and $\neg Sm (0.4) = \neg Sm(0.6)= \neg Sm(0.8)=1$. 

Consequently,
\begin{equation} 
 \neg Sm \left(\frac{|X_{Sport} \cap [x]_\mathcal{R}|}{|[x]_\mathcal{R}|}\right)= \begin{cases} 0 & \mbox{ if } x \in C_1 , \\ 0.25 & \mbox{ if } x \in C_6, \\ 0.75 & \mbox{ if } x \in C_2\\ 1 & \mbox{ if } x \in C_3 \cup C_4 \cup C_5. 
 \end{cases}
\end{equation}

Then, the positive, negative and boundary regions induced by $(0.8,0.2)$ and $\neg Sm$ are the following:
\begin{description}
    \item $POS^{\neg Sm}_{(0.8,0.2)}(X_{Sport})=\left\{x \in U \ | \ \neg Sm \left(\dfrac{|
    X_{Sport} \cap [x]_\mathcal{R}|}{|[x]_\mathcal{R}|}\right) \geq 0.8\right\}=C_3 \cup C_4 \cup C_5$,
    \vspace{0.2cm}
     \item $NEG^{\neg Sm}_{(0.8,0.2)}(X_{Sport})=\left\{x \in U \ | \ \neg Sm \left(\dfrac{|
    X _{Sport}\cap [x]_\mathcal{R}|}{|[x]_\mathcal{R}|}\right) \leq 0.2\right\}=C_1$,
      \vspace{0.2cm}
    \item $BND^{\neg Sm}_{(0.8,0.2)}(X_{Sport})=\left\{x \in U \ | \ 0.2 < \neg Sm \left(\dfrac{|
    X_{Sport} \cap [x]_\mathcal{R}|}{|[x]_\mathcal{R}|}\right) < 0.8\right\}= C_2 \cup C_6$.
\end{description}
The three regions lead the following decisions. Firstly, we choose to provide the news about the sport to the communities  $C_3$, $C_4$, and $C_5$ that form the positive region, considering that \emph{these contains many users interested in sport} with a degree that we think is enough high ($\geq 0.8$). Moreover, we require further analysis on the communities $C_2$ and $C_6$ forming the boundary region, before providing news about sports. For example, we could evaluate the interests of their users in the future or once new users join them. Finally, we surely do not provide sports news to $C_1$ because we think that not enough of its users are interested in  sports topics, indeed we consider the degree to which \emph{many users of $C_1$ are interested in sports} low ($\leq 0.2$).

\subsection{Linguistic Rough Sets}
Definition \ref{def:ee}  also leads to a 
novel generalization of Pawlak rough sets.

\begin{definition} \label{def:linguisticRS}
Let $Ev \in \mathcal{E}$, the $(\alpha,\beta)$-linguistic rough set of $X$ determined by $\mathcal{R}$ and $Ev$ is the pair $(\mathcal{L}_{(\alpha,\beta)}^{Ev}(X),\mathcal{U}_{(\alpha,\beta)}^{Ev}(X)),$ where
\begin{multline*}\mathcal{L}_{(\alpha,\beta)}^{Ev}(X)=\left\{x \in X \ | \ Ev\left(\dfrac{|[x]_{\mathcal{R}} \cap X|}{|[x]_{\mathcal{R}}|} \right) \geq \alpha \right\} \ \mbox{ and  } \ \\ \mathcal{U}_{(\alpha,\beta)}^{Ev}(X)=\left\{x \in X \ | \ Ev\left(\dfrac{|[x]_{\mathcal{R}} \cap X|}{|[x]_{\mathcal{R}}|} \right) > \beta \right\}.
\end{multline*}
$\mathcal{L}_{(\alpha,\beta)}^{Ev}(X)$ and $\mathcal{U}_{(\alpha,\beta)}^{Ev}(X)$ are respectively called $(\alpha,\beta)$-\emph{linguistic lower and upper approximation} of $X$ determined by $\mathcal{R}$ and $Ev$.
\end{definition}
Equivalently, by Definition \ref{def:ee}, we get 
\begin{center}
$\mathcal{L}^{Ev}_{(\alpha,\beta)}(X)=POS_{(\alpha,\beta)}^{Ev}(X)$ \ and \ $\mathcal{U}^{Ev}_{(\alpha,\beta)}(X)=POS_{(\alpha,\beta)}^{Ev}(X) \cup BND_{(\alpha,\beta)}^{Ev}(X)$.
\end{center}

 Let $x \in U$, the value $Ev\left(\dfrac{|[x]_{\mathcal{R}} \cap X|}{|[x]_{\mathcal{R}}|} \right)$ is viewed as the degree of confidence expressing \emph{how much we can trust that
$x$ belongs to $X$}.

The following is an illustrative example. 


\begin{example}
Consider Example \ref{esempio}. In terms of generalized rough sets, $X_{Sport}$ can be approximated by the $(0.8,0.2)$-linguistic rough set 
\begin{multline*}
(\mathcal{L}_{(0.8,0.2)}^{\neg Sm}(X_{Sport}), \mathcal{U}_{(0.8,0.2)}^{\neg Sm}(X_{Sport}))=(C_3 \cup C_4 \cup C_5, C_2 \cup C_3 \cup C_4 \cup C_5 \cup C_6) = \\ (\{u_{11}, \ldots, u_{25}\},\{u_{6}, \ldots, u_{32}\})\}. 
\end{multline*}
Each element $x \in U$ 
is associated with the value $\neg Sm\left(\dfrac{|X_{Sport} \cap [x]_{\mathcal{R}}}{[x]_{\mathcal{R}}}\right)$, which is understood as the degree of confidence expressing \emph{how much we can trust that $x$ is interested in sports}.
\end{example}

\section{Connection with TWD methods} \label{sec:4}
In this section, we find a link between the TWD methods based on probabilistic rough sets  and evaluative expressions. 
In particular, we fix
a finite universe $U$, a subset $X$ of $U$, an equivalence relation $\mathcal{R}$ on $U$, and a pair of thresholds $(\alpha,\beta)$ such that $0 \leq \alpha < \beta \leq 1$,  and we aim to determine for each evaluative expression $Ev$, the class of all pairs of thresholds like $(\alpha',\beta')$ so that $\mathcal{T}_{(\alpha',\beta')}(X)$ coincides with  $\mathcal{T}_{(\alpha,\beta)}^{Ev}(X)$. 

In this paper, we confine to  the class $\mathcal{E}^+ \subset \mathcal{E}$, which is made of all extensions that are increasing functions, i.e. let $Ev \in \mathcal{E}$, $Ev \in \mathcal{E}^+$ if and only if ``$Ev(x) \leq Ev(y)$ for each $x, y \in [0,1]$ such that $x \leq y$". Examples of evaluative expressions so that their extension is an increasing function, are \emph{not small}, \emph{very big}, and \emph{extremely big} (see \eqref{eq:nSm}, \eqref{eq:vBig}, and \eqref{eq:eBig}). 
However, there exist evaluative expressions like \emph{small} that are represented by a decreasing function and others like \emph{medium} that are represented by a non-monotonic function.  

In order to obtain the results of this section, we need to define the values $\alpha^{Ev}_1$, $\alpha^{Ev}_2$, $\beta^{Ev}_1$, and $\beta^{Ev}_2$, which are associated with 
$\mathcal{T}^{Ev}_{(\alpha,\beta)}(X)$, where $Ev \in \mathcal{E}$.

 
\begin{definition}
Let $Ev \in \mathcal{E}$. If $POS_{(\alpha,\beta)}^{Ev}(X), NEG_{(\alpha,\beta)}^{Ev}(X)$, $BND_{(\alpha,\beta)}^{Ev}(X) \neq \emptyset$, then we put
\begin{enumerate} \label{def:alpha}
 \item [(i)] $\alpha^{Ev}_1= \max\left\{\dfrac{|X \cap [x]_\mathcal{R}|}{|[x]_\mathcal{R}|} \ | \ x \in BND^{Ev}_{(\alpha,\beta)}(X)\right\}$,
    \item [(ii)] $\alpha^{Ev}_2= \min\left\{\dfrac{|X \cap [x]_\mathcal{R}|}{|[x]_\mathcal{R}|} \ | \ x \in POS^{Ev}_{(\alpha,\beta)}(X)\right\}$,
      \item [(iii)] $\beta^{Ev}_1= \max\left\{\dfrac{|X \cap [x]_\mathcal{R}|}{|[x]_\mathcal{R}|} \ | \ x \in NEG^{Ev}_{(\alpha,\beta)}(X)\right\},$ 
    \item [(iv)] $\beta^{Ev}_2=  \min\left\{\dfrac{|X \cap [x]_\mathcal{R}|}{|[x]_\mathcal{R}|} \ | \ x \in BND^{Ev}_{(\alpha,\beta)}(X)\right\}.$
\end{enumerate}
\end{definition}
\begin{example} \label{esempio2}
Consider the universe $U$, its subset $X_{Sport}$, and the pair of thresholds $(\alpha,\beta)$ that are defined by Example \ref{esempio}.  Then, the corresponding positive, negative, and boundary regions are the following: 
\begin{multline*}
POS_{(0.8,0.2)}^{\neg Sm}(X_{Sport})=C_3 \cup C_4 \cup C_5, \  NEG^{\neg Sm}_{(0.8,0.2)}(X_{Sport})=C_1, \ \mbox{ and } \\ 
BND^{\neg Sm}_{(0.8,0.2)}(X_{Sport})=C_2 \cup C_6.
\end{multline*}
Hence, by \eqref{eq:ex}, we get \footnote{Recall that the equivalence classes of $\{[x]_{\mathcal{R}} \ | \ x \in U\}$ are the sets $C_1,C_2,C_3,C_4,C_5$, and $C_6$.} 
\begin{multline*}
\left\{\dfrac{|X_{Sport} \cap [x]_\mathcal{R}|}{|[x]_\mathcal{R}|} \ | \ x \in POS^{Ev}_{(\alpha,\beta)}(X_{Sport})\right\}= \\ \left\{\dfrac{|X_{Sport} \cap C_3|}{|C_3|},\dfrac{|X_{Sport} \cap C_4|} 
 {|C_4|},  \dfrac{|X_{Sport} \cap C_5|}{|C_5|}\right\}=\{0.4, 0.6, 0.8\};
 \end{multline*}
\begin{multline*}
\left\{\dfrac{|X_{Sport} \cap [x]_\mathcal{R}|}{|[x]_\mathcal{R}|} \ | \ x \in NEG^{Ev}_{(\alpha,\beta)}(X_{Sport})\right\}=\left\{\dfrac{|X_{Sport} \cap C_1|}{|C_1|}\right\}=\{0\};
\end{multline*}
\begin{multline*}
\left\{\dfrac{|X_{Sport} \cap [x]_{\mathcal{R}}|}{|[x]_\mathcal{R}|} \ | \ x \in BND^{Ev}_{(\alpha,\beta)}(X_{Sport})\right\}=\left\{\dfrac{|X_{Sport} \cap C_2|}{|C_2|},\dfrac{|X_{Sport} \cap C_6|}{|C_6|}\right\} \\ =  \{0.2, 0.14\}.
\end{multline*}
Finally, by Definition \ref{def:alpha},  $\alpha_1^{\neg Sm}  =\max\{0.14, 0.2\}=0.2$, $\alpha_2^{\neg Sm}=\min\{0.4,0.6,0.8\}=0.4$, $\beta_1^{\neg Sm}=\max\{0\}=0$, and $\beta_2^{\neg Sm}=\min\{0.14,0.2\} \\ =0.14$.
\end{example}
If $Ev$ is an increasing function, namely $Ev \in \mathcal{E}^+$, then we can order $\beta_1^{Ev}$, $\beta_2^{Ev}$, $\alpha_1^{Ev}$, and $\alpha_2^{Ev}$ as shown in the following proposition. 
\begin{proposition} \label{pro:1}
Let $Ev \in \mathcal{E}^+$. If $POS_{(\alpha,\beta)}^{Ev}(X), NEG_{(\alpha,\beta)}^{Ev}(X)$, $BND_{(\alpha,\beta)}^{Ev}(X)$ $\neq \emptyset$, then  $0 \leq \beta_1^{Ev} < \beta_2^{Ev} \leq \alpha_1^{Ev} < \alpha_2^{Ev} \leq 1$.
\end{proposition}
\begin{proof}
By Definition \ref{def:alpha}, it is trivial that
$0 \leq \beta_1^{Ev}, \beta^{Ev}_2, \alpha^{Ev}_1, \alpha^{Ev}_2 \leq 1$.
\begin{description}
\item [$(\beta_1^{Ev} < \beta_2^{Ev})$.] 
By Definition \ref{def:alpha} ((iii) and (iv)), $\beta_1^{Ev}=\dfrac{|X \cap [x_1]_{\mathcal{R}}|}{|[x_1]_{\mathcal{R}}|}$ with $x_1 \in NEG_{(\alpha,\beta)}^{Ev}(X)$ and $\beta^{Ev}_2=\dfrac{|X \cap [x_2]_{\mathcal{R}}|}{|[x_2]_{\mathcal{R}}|}$ with $x_2 \in BND_{(\alpha,\beta)}^{Ev}(X)$.  Then, by Definition \ref{def:ee} ((ii) and (iii)),  $Ev(\beta^{Ev}_1) \leq \beta$ and $\beta < Ev(\beta_{2}^{Ev}) < \alpha$. Hence, $Ev(\beta^{Ev}_1) < Ev(\beta^{Ev}_2)$. Thus, considering that $Ev$ is an increasing function, we can conclude that $\beta^{Ev}_1 < \beta^{Ev}_2$.
\item [($\alpha^{Ev}_1 < \alpha^{Ev}_2$).] 
By Definition \ref{def:alpha} ((i) and (ii)), $\alpha_1^{Ev}=\dfrac{|X \cap [x_1]_{\mathcal{R}}|}{|[x_1]_{\mathcal{R}}|}$ with $x_1 \in BND_{(\alpha,\beta)}^{Ev}(X)$ and $\alpha^{Ev}_2=\dfrac{|X \cap [x_2]_{\mathcal{R}}|}{|[x_2]_{\mathcal{R}}|}$ with $x_2 \in POS_{(\alpha,\beta)}^{Ev}(X)$. Thus, by Definition \ref{def:ee} ((iii) and (i)), $\beta < Ev(\alpha^{Ev}_1) < \alpha$ and $Ev(\alpha^{Ev}_2) \geq \alpha$. Thus,  $Ev(\alpha^{Ev}_1) < Ev(\alpha^{Ev}_2)$. Hence, considering that   $Ev$ is an increasing function, $\alpha^{Ev}_1 < \alpha^{Ev}_2$. 
\item [($\beta^{Ev}_2 \leq \alpha^{Ev}_1$).] By Definition \ref{def:alpha} ((i) and (iii)),  $\alpha_1^{Ev}=\dfrac{|X \cap [x_1]_{\mathcal{R}}|}{|[x_1]_{\mathcal{R}}|}$ with $x_1 \in BND_{(\alpha,\beta)}^{Ev}(X)$ and $\beta_2^{Ev}=\dfrac{|X \cap [x_2]_{\mathcal{R}}|}{|[x_2]_{\mathcal{R}}|}$ with $x_2 \in BND_{(\alpha,\beta)}^{Ev}(X)$. Therefore, since $x_1, x_2 \in BND_{(\alpha,\beta)}^{Ev}(X)$, $\beta^{Ev}_2 \leq \alpha^{Ev}_1$ clearly holds.
\end{description}
\end{proof}
\begin{example}
 In Example \ref{esempio2}, we have shown that $\alpha_1^{\neg Sm}=0.2$, $\alpha_2^{\neg Sm}=0.4$, $\beta_1^{\neg Sm}=0$, and $\beta_2^{\neg Sm} =0.14$. Then, according to Proposition \ref{pro:1},  $0 \leq \beta_1^{Ev} < \beta_2^{Ev} \leq \alpha_1^{Ev} < \alpha_2^{Ev} \leq 1$.
\end{example}
The next theorems show that the three regions generated by  $Ev \in \mathcal{E}^+$ can be also obtained by using the probabilistic approach and changing the initial thresholds. We separately analyze the following cases: all three regions are non-empty (Theorem \ref{teo:connection}) and only one of the three regions is empty (Theorems \ref{teo:2}-\ref{teo:4}). The remaining case where only one region is non-empty (namely, one of the three regions coincides with the universe) is omitted because not significant. 
\begin{theorem} \label{teo:connection}
Let $Ev \in \mathcal{E}^+$ such that $POS_{(\alpha,\beta)}^{Ev}(X), NEG_{(\alpha,\beta)}^{Ev}(X)$, $BND_{(\alpha,\beta)}^{Ev}(X)$ $\neq \emptyset$ and let $\alpha', \beta' \in [0,1]$ with $\beta'< \alpha'$. Then, 
\begin{center}
$\mathcal{T}_{(\alpha,\beta)}^{Ev}(X)=\mathcal{T}_{(\alpha',\beta')}(X)$
if and only if $\alpha' \in (\alpha^{Ev}_1,\alpha^{Ev}_2]$ and $\beta' \in [\beta^{Ev}_1,\beta^{Ev}_2)$.
\end{center}
\end{theorem}

\begin{remark}
Before to prove Theorem \ref{teo:connection}, let us represent the intervals that contain $\alpha'$ and $\beta'$ (i.e. the values for generating $\mathcal{T}_{(\alpha,\beta)}^{Ev}(X)$ with probabilistic rough sets) by Figure \ref{fig:thresholds}. 
By Definition \ref{def:alpha}, these intervals separates  $POS_{(\alpha,\beta)}^{Ev}(X)$ from $BND_{(\alpha,\beta)}^{Ev}(X)$ and $NEG_{(\alpha,\beta)}^{Ev}(X)$ from $BND_{(\alpha,\beta)}^{Ev}(X)$. More precisely, the value $Ev\left(\dfrac{|X \cap [x]_{\mathcal{R}}|}{|[x]_{\mathcal{R}}|}\right)$ belongs to $[0,\beta^{Ev}_1]$ when $x \in NEG_{(\alpha,\beta)}^{Ev}(X)$, to $[\beta^{Ev}_2, \alpha^{Ev}_1]$ when $x \in BND^{Ev}_{(\alpha,\beta)}(X)$, and to $[\alpha^{Ev}_2,1]$ when $x \in POS_{(\alpha,\beta)}^{Ev}(X)$.
   \begin{figure}[h]
	\begin{center}
		\[\scalebox{1}{\begin{tikzpicture}
\node (1) at (0,0){};
\node (2) at (7,0){};
\node (11) at (0.1,0){$\bullet$};
\node (30) at (0,0.5){$0$};
\node (4) at (1.3,0){$[$};
\node (40) at (1.3,0.5){$\beta_1^{Ev}$};
\node (6) at (2.7,0){$)$};
\node (60) at (2.7,0.5){$\beta^{Ev}_2$};
\node (7) at (4.3,0){$($};
\node (70) at (4.3,0.5){$\alpha^{Ev}_1$};
\node (9) at (5.7,0){$]$};
\node (90) at (5.7,0.5){$\alpha^{Ev}_2$};
\node (10) at (6.9,0){$\bullet$};
\node (100) at (6.9,0.5){$1$};
\node (12) at (1.2,0){};
\node (13) at (2.8,0){};
\node (14) at (4.2,0){};
\node (15) at (5.8,0){};
\draw[-] (1) to node [midway]{}(2);
\draw[-, very thick] (12) to node [midway]{}(13);
\draw[-, very thick] (14) to node [midway]{}(15);
			\end{tikzpicture}}\]
	\end{center}
	\caption{Intervals $[\beta_1^{Ev},\beta_2^{Ev})$  and $(\alpha_1^{Ev},\alpha_2^{Ev}]$.} \label{fig:thresholds}
\end{figure}
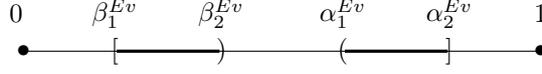
\end{remark}

\begin{proof}
$(\Leftarrow).$
Let $\alpha' \in (\alpha^{Ev}_1,\alpha^{Ev}_2]$ and let $\beta' \in [\beta^{Ev}_1,\beta^{Ev}_2)$, we need to prove that $POS_{(\alpha,\beta)}^{Ev}(X)=POS_{(\alpha',\beta')}(X)$, $NEG_{(\alpha,\beta)}^{Ev}(X)=NEG_{(\alpha',\beta')}(X)$, and $BND_{(\alpha,\beta)}^{Ev}$ $(X)=BND_{(\alpha',\beta')}(X)$. 

\begin{description}
\item [($POS_{(\alpha,\beta)}^{Ev}(X)=POS_{(\alpha',\beta')}(X)$).] Let $\bar{x} \in POS^{Ev}_{(\alpha, \beta)}(X)$, then $\dfrac{|X \cap [\bar{x}]_\mathcal{R}|}{|[\bar{x}]_\mathcal{R}|} \geq \alpha^{Ev}_2$ from Definition \ref{def:alpha} (ii). Moreover, $\alpha' \leq \alpha^{Ev}_2$ because $\alpha' \in (\alpha^{Ev}_1,\alpha^{Ev}_2]$. Consequently, $\dfrac{|X \cap [\bar{x}]_\mathcal{R}|}{|[\bar{x}]_\mathcal{R}|} \geq \alpha'$. Finally, $\bar{x} \in POS_{(\alpha',\beta')}(X)$ from Definition \ref{def:PRS} (i). 

Let $\bar{x} \in POS_{(\alpha',\beta')}(X)$, then $\dfrac{|X \cap [\bar{x}]_\mathcal{R}|}{|[\bar{x}]_\mathcal{R}|} \geq \alpha'$ from Definition \ref{def:PRS} (i). By the previous inequality and $\alpha' > \alpha^{Ev}_1$, we get  $\dfrac{|X \cap [\bar{x}]_\mathcal{R}|}{|[\bar{x}]_\mathcal{R}|} > \alpha^{Ev}_1$. Hence, considering that $\alpha^{Ev}_1$ is the maximum of $\left\{\dfrac{|X \cap [x]_{\mathcal{R}}|}{|[x]_{\mathcal{R}}|} \ | \ x \in BND_{(\alpha,\beta)}^{Ev}(X)\right\}$ (see Definition \ref{def:alpha}(i)), we are sure that $\bar{x} \notin BND_{(\alpha,\beta)}^{Ev}(X)$. Furthermore, $\dfrac{|X \cap [\bar{x}]_\mathcal{R}|}{|[\bar{x}]_\mathcal{R}|} > \alpha^{Ev}_1$ and $\beta_1^{Ev}< \alpha^{Ev}_1$ (see Proposition \ref{pro:1}) imply that  $\dfrac{|X \cap [\bar{x}]_\mathcal{R}|}{|[\bar{x}]_\mathcal{R}|} > \beta^{Ev}_1$. Thus, considering that $\beta^{Ev}_1$ is the maximum of $\left\{\dfrac{|X \cap [x]_{\mathcal{R}}|}{|[x]_{\mathcal{R}}|} \ | \ x \in NEG_{(\alpha,\beta)}^{Ev}(X)\right\}$ (see Definition \ref{def:alpha}(iii)), we have $\bar{x} \notin NEG_{(\alpha,\beta)}^{Ev}(X)$. Ultimately, by \eqref{eq:tr}, $\bar{x} \in POS_{(\alpha,\beta)}^{Ev}(X)$. 

\item [($BND_{(\alpha,\beta)}^{Ev}(X)=BND_{(\alpha',\beta')}(X)$).] Let $\bar{x} \in BND_{(\alpha,\beta)}^{Ev}(X)$. By Definition \ref{def:alpha} ((i) and (iv)),  $\beta_2^{Ev} \leq \dfrac{|X \cap [\bar{x}]_{\mathcal{R}}|}{|[\bar{x}]_{\mathcal{R}}|} \leq \alpha_1^{Ev}$. Moreover, by hypothesis, $\beta' < \beta_2^{Ev}$  and $\alpha' > \alpha_1^{Ev}$. Thus, we can conclude that $\beta' < \dfrac{|X \cap [\bar{x}]_{\mathcal{R}}|}{|[\bar{x}]_{\mathcal{R}}|} < \alpha'$, namely $\bar{x} \in BND_{(\alpha',\beta')}(X)$ from Definition \ref{def:PRS} (iii). 

Let $\bar{x} \in BND_{(\alpha',\beta')}(X)$, then $\beta' < \dfrac{|X \cap [\bar{x}]_{\mathcal{R}}|}{|[\bar{x}]_{\mathcal{R}}|} < \alpha'$ from Definition \ref{def:PRS} (iii). By hypothesis,  $\beta^{Ev}_1 \leq \beta'$ and $\alpha' \leq \alpha^{Ev}_2$. Hence, we know that $\beta^{Ev}_1 < \dfrac{|X \cap [\bar{x}]_{\mathcal{R}}|}{|[\bar{x}]_{\mathcal{R}}|} < \alpha^{Ev}_2$. By Definition \ref{def:alpha} (iii),  $\beta^{Ev}_1 < \dfrac{|X \cap [\bar{x}]_{\mathcal{R}}|}{|[\bar{x}]_{\mathcal{R}}|}$ implies that $\bar{x} \notin NEG_{(\alpha,\beta)}^{Ev}(X)$. Furthermore, by Definition \ref{def:alpha} (ii), $\dfrac{|X \cap [\bar{x}]_{\mathcal{R}}|}{|[\bar{x}]_{\mathcal{R}}|} < \alpha^{Ev}_2$ implies that $\bar{x} \notin POS^{Ev}_{(\alpha,\beta)}(X)$. Ultimately, by \eqref{eq:tr}, $\bar{x} \in BND^{Ev}_{(\alpha,\beta)}(X)$.

\item [($NEG_{(\alpha,\beta)}^{Ev}(X)=NEG_{(\alpha',\beta')}(X)$).] We have previously shown that $POS_{(\alpha,\beta)}^{Ev}(X)$ $=POS_{(\alpha',\beta')}(X)$ and $BND_{(\alpha,\beta)}^{Ev}(X)=BND_{(\alpha',\beta')}(X)$. So, by \eqref{eq:tr1} and \eqref{eq:tr}, it is clear that  $NEG_{(\alpha,\beta)}^{Ev}(X)=NEG_{(\alpha',\beta')}(X)$.

\end{description}

$(\Rightarrow).$ Let $\mathcal{T}_{(\alpha,\beta)}^{Ev}(X)=\mathcal{T}_{(\alpha',\beta')}(X)$, we intend to prove that $\beta^{Ev}_1 \leq \beta' < \beta^{Ev}_2$ and $\alpha^{Ev}_1 < \alpha' \leq \alpha^{Ev}_2$. 
\begin{description}
\item [($\alpha' \leq \alpha^{Ev}_2$).]  
Let $x_2 \in U$ such that $\alpha^{Ev}_2 = \dfrac{|X \cap [x_2]_{\mathcal{R}}|}{|[x_2]_{\mathcal{R}}|}$. By Definition \ref{def:alpha} (ii), $x_2 \in POS_{(\alpha,\beta)}^{Ev}(X)$. Hence,  $\alpha' > \alpha^{Ev}_2$ means that $\dfrac{|X \cap [x_2]_{\mathcal{R}}|}{|[x_2]_{\mathcal{R}}|}<\alpha'$. Thus, $x_2 \notin POS_{(\alpha',\beta')}(X)$ from Definition \ref{def:PRS} (i). This contradicts that $POS^{Ev}_{(\alpha,\beta)}(X)=POS_{(\alpha',\beta')}(X)$. Thus, it must be true that $\alpha' \leq \alpha^{Ev}_2$.
\item [($\alpha^{Ev}_1 < \alpha'$).] Let $x_1 \in U$ such that $\alpha^{Ev}_1=\dfrac{|X \cap [x_1]_{\mathcal{R}}|}{|[x_1]_{\mathcal{R}}|}$. By Definition \ref{def:alpha} (i), $x_1 \in BND^{Ev}_{(\alpha,\beta)}(X)$. If $\alpha^{Ev}_1 \geq \alpha'$, then $\dfrac{|X \cap [x_1]_{\mathcal{R}}|}{|[x_1]_{\mathcal{R}}|} \geq \alpha'$. So, $x_1 \in POS_{(\alpha',\beta')}(X)$ from Definition \ref{def:PRS} (i). This contradicts that $POS^{Ev}_{(\alpha,\beta)}(X)=POS_{(\alpha',\beta')}(X)$. Thus, it must be true that $\alpha^{Ev}_1 < \alpha'$.
\item [$(\beta_1^{Ev} \leq \beta')$.] Let $x_1 \in U$ such that $\beta_1^{Ev}=\dfrac{|X \cap [x_1]_{\mathcal{R}}|}{|[x_1]_{\mathcal{R}}|}$. By Definition \ref{def:alpha}(iii), $x_1 \in NEG^{Ev}_{(\alpha,\beta)}(X)$. If $\beta_1^{Ev} > \beta'$, then $\dfrac{|X \cap [x_1]_{\mathcal{R}}|}{|[x_1]_{\mathcal{R}}|} > \beta'$, which implies that $x_1 \notin NEG_{(\alpha',\beta')}(X)$  from Definition \ref{def:PRS}(ii). This contradicts that $NEG^{Ev}_{(\alpha,\beta)}(X)=NEG_{(\alpha',\beta')}(X)$. Thus, it must be true that $\beta_1^{Ev} \leq \beta'$.
\item [$(\beta' < \beta^{Ev}_2)$.] Let $x_2 \in U$ such that $\beta^{Ev}_2=\dfrac{|X \cap [x_2]_{\mathcal{R}}|}{|[x_2]_{\mathcal{R}}|}$. By Definition \ref{def:alpha}(iv), $x_2 \in BND^{Ev}_{(\alpha,\beta)}(X)$. Also, if $\beta' \geq \beta^{Ev}_2$, then $\dfrac{|X \cap [x_2]_{\mathcal{R}}|}{|[x_2]_{\mathcal{R}}|} \leq \beta'$, which implies that $x_2 \in NEG_{(\alpha',\beta')}(X)$ from Definition \ref{def:PRS} (ii). This contradicts that $BND^{Ev}_{(\alpha,\beta)}(X)=BND_{(\alpha',\beta')}(X)$. Thus, it must be true that $\beta' < \beta^{Ev}_2$.

\end{description}

\end{proof}
\begin{example}
Consider Example \ref{esempio}, $\neg Sm$ is an increasing function and all the three regions of $\mathcal{T}_{(0.8,0.2)}^{\neg Sm}(X_{Sport})$ are non-empty. In Example \ref{esempio2}, we have found that $\alpha_1^{\neg Sm}=0.2$, $\alpha_2^{\neg Sm}=0.4$, $\beta_1^{\neg Sm}=0$, and $\beta_2^{\neg Sm} =0.14$. Therefore, according to Theorem \ref{teo:connection}, we get
$\mathcal{T}^{\neg Sm}_{(0.8,0.2)}(X_{Sport})=\mathcal{T}_{(\alpha',\beta')}(X_{Sport})$  for each $(\alpha',\beta')$ such that $\alpha' \in (0.2,0.4]$ and $\beta' \in [0,0.14)$

For example, we can easily verify that $\mathcal{T}^{\neg Sm}_{(0.8,0.2)}(X_{Sport})=\mathcal{T}_{(0.3,0.1)}(X_{Sport})$. Indeed, by \eqref{eq:ex} and by Definition \ref{def:PRS}, 
\begin{itemize}
\item $POS_{(0.3,0.1)}(X_{Sport})=\left\{x \in U \ | \ \dfrac{|X \cap [x]_{\mathcal{R}}|}{|[x]_{\mathcal{R}}|} \geq 0.3\right\}=C_3 \cup C_4 \cup C_5$,  
\item $NEG_{(0.3,0.1)}(X_{Sport})=\left\{x \in U \ | \ \dfrac{|X \cap [x]_{\mathcal{R}}|}{|[x]_{\mathcal{R}}|} \leq 0.1\right\}=C_1$, and 
\item $BND_{(0.3,0.1)}(X_{Sport})=\left\{x \in U \ | \ 0.1 < \dfrac{|X \cap [x]_{\mathcal{R}}|}{|[x]_{\mathcal{R}}|} < 0.3\right\}=C_2 \cup C_6$. 
\end{itemize} 
\end{example}

By Theorem \ref{teo:connection}, we can connect linguistic rough sets with classical rough sets. 
More precisely, the following corollary holds. 

\begin{corollary}
Let $Ev \in \mathcal{E}^+$ with $POS_{(\alpha,\beta)}^{Ev}(X), NEG_{(\alpha,\beta)}^{Ev}(X)$, $BND_{(\alpha,\beta)}^{Ev}(X)$ $\neq \emptyset$. Then,
\begin{center}
$(\mathcal{L}_{(\alpha,\beta)}^{Ev}(X), \mathcal{U}_{(\alpha,\beta)}^{Ev}(X))=(\mathcal{L}(X), \mathcal{U}(X))$ \footnote{Recall that $(\mathcal{L}_{(\alpha,\beta)}^{Ev}(X), \mathcal{U}_{(\alpha,\beta)}^{Ev}(X))$ is the linguistic rough set of $X$ given by Definition \ref{def:linguisticRS}  and $(\mathcal{L}(X),\mathcal{U}(X))$ is the  classical rough set of $X$ given by Eq. \eqref{eq:RS}.} if and only if $\beta^{Ev}_1=0$ and $\alpha^{Ev}_2=1$.
\end{center}
\end{corollary}
\begin{proof}

$(\Rightarrow).$ Suppose that $(\mathcal{L}_{(\alpha,\beta)}^{Ev}(X),  \mathcal{U}_{(\alpha,\beta)}^{Ev}(X))$ is the rough set of $X$.  Then, by \eqref{eq:RS}, let $x \in U$,
$x \in POS_{(\alpha,\beta)}^{Ev}(X)$ 
if and only if $[x]_{\mathcal{R}} \subseteq X$. 
The latter means that $\dfrac{|X \cap [x]_{\mathcal{R}}|}{|[x]_{\mathcal{R}}|}=1$ for each $x \in POS_{(\alpha,\beta)}^{Ev}(X)$. Hence, By Definition \ref{def:alpha} (ii),
$\alpha_2^{Ev}=1$. 

By \eqref{eq:RS}, let $x \in U$,
$x \in POS_{(\alpha,\beta)}^{Ev}(X) \cup BND_{(\alpha,\beta)}^{Ev}(X)$ 
if and only if $[x]_{\mathcal{R}} \cap X \neq \emptyset$. Since $POS_{(\alpha,\beta)}^{Ev}(X) \cup BND_{(\alpha,\beta)}^{Ev}(X)=U \setminus NEG_{(\alpha,\beta)}^{Ev}(X)$, we know that  $\dfrac{|X \cap [x]_{\mathcal{R}}|}{|[x]_{\mathcal{R}}|}=0$ for each $x \in NEG_{(\alpha,\beta)}^{Ev}(X)$. Finally, by Definition \ref{def:alpha} (iii), $\beta^{Ev}_1=0$.

$(\Leftarrow)$. Suppose that $\alpha^{Ev}_2=1$ and $\beta^{Ev}_1=0$. Trivially, $\alpha^{Ev}_2 \in (\alpha^{Ev}_1, \alpha^{Ev}_2]$ and $\beta^{Ev}_1 \in [\beta^{Ev}_1, \beta^{Ev}_2)$.  Then, by Theorem \ref{teo:connection},
$(POS_{(\alpha,\beta)}^{Ev}(X),$ $POS_{(\alpha,\beta)}^{Ev}(X) \cup BND_{(\alpha,\beta)}^{Ev}(X))=(POS_{(1,0)}(X), POS_{(1,0)}(X) \cup BND_{(1,0)}(X))$. Moreover, by Remark \ref{rem:1}, $(\mathcal{L}_{(1,0)}(X),\mathcal{U}_{(1,0)}(X))=(POS_{(1,0)}(X),$ $POS_{(1,0)}(X) \cup BND_{(1,0)}(X))$ coincides with the rough set $(\mathcal{L}(X),\mathcal{U}(X))$ given by \eqref{eq:RS}.  
\end{proof}
\begin{example} \label{esempio3}
Consider the universe $U=\{u_1, \ldots, u_{20}\}$ and the evaluative expression \emph{Very big}, which is modelled by \eqref{eq:vBig}. We suppose that $U$ is partitionated into three equivalence classes: $C_1=\{u_1, \ldots, u_5\}$, $C_2=\{u_6, \ldots, u_{10}\}$, and $C_3=\{u_{11}, \ldots, u_{20}\}$. If $X=\{u_6, \ldots, u_{19}\}$, we can simply prove that $(\mathcal{L}^{BiVe}_{(0.7,0.3)}(X), \mathcal{E}^{BiVe}_{(0.7,0.3)}(X))=(\mathcal{L}(X),\mathcal{U}(X))$. Indeed, we get
\begin{equation*} 
\dfrac{|X \cap [x]_\mathcal{R}|}{|[x]_\mathcal{R}|}=\begin{cases}
0 & \mbox{ if } x \in C_1,\\
0.9 & \mbox{ if } x \in C_3, \\
1 & \mbox{ if } x \in C_2.
\end{cases} \mbox{ and } 
 Bi Ve \left(\frac{|X \cap [x]_\mathcal{R}|}{|[x]_\mathcal{R}|}\right)= \begin{cases} 0 & \mbox{ if } x \in C_1 , \\  0.59 & \mbox{ if } x \in C_3,\\ 1 & \mbox{ if } x \in C_2. 
 \end{cases}
\end{equation*}
Thus, 
\begin{description}
\item $POS_{(0.7,0.3)}^{BiVe}(X)= \{C_i   \ | \ i \in \{1,2,3\} \mbox{ and } BiVe\left(\dfrac{|X \cap C_i|}{|C_i|}\right) \geq 0.7\}=C_2$;
\item $NEG_{(0.7,0.3)}^{BiVe}(X)= \{C_i   \ | \ i \in \{1,2,3\} \mbox{ and } BiVe\left(\dfrac{|X \cap C_i|}{|C_i|}\right) \leq 0.3\}=C_1$;
\item $BND_{(0.7,0.3)}^{BiVe}(X)= \{C_i   \ | \ i \in \{1,2,3\} \mbox{ and } 0.3 < BiVe\left(\dfrac{|X \cap C_i|}{|C_i|}\right) < 0.7\}=C_3$.
\end{description}
Also, by Definition \ref{def:alpha}, $\beta_1^{BiVe}=0$, $\alpha_2^{BiVe}=1$, $\alpha_1^{BiVe}=\beta_2^{BiVe}=0.9$. Since the hypothesis of the previous corollary is satisfied, we expect that \\ $(\mathcal{L}(X),\mathcal{U}(X))=(\mathcal{L}^{BiVe}_{(0.7,0.3)}(X), \mathcal{U}^{BiVe}_{(0.7,0.3)}(X))=(C_2, C_2 \cup C_3)$. We can immediately verify that this is true: $\mathcal{L}(X)=C_2$ because $C_2$ is the unique class among $C_1, C_2$, and $C_3$ that is included in $X$; moreover, $\mathcal{U}(X)=C_2 \cup C_3$ because $X \cap C_2 \neq \emptyset$ and $X \cap C_3 \neq \emptyset$, while $X \cap C_1=\emptyset$. 
\end{example}

We are now going to deal with the cases where one of $BND^{Ev}_{(\alpha,\beta)},POS^{Ev}_{(\alpha,\beta)},$  $NEG^{Ev}_{(\alpha,\beta)}$ is empty.

\begin{theorem} \label{teo:2}
Let $Ev \in \mathcal{E}^+$ such that $BND_{(\alpha,\beta)}^{Ev} =\emptyset$ and $POS_{(\alpha,\beta)}^{Ev}, NEG_{(\alpha,\beta)}^{Ev} \neq \emptyset$. Let $\alpha', \beta' \in [0,1]$ such that $\beta' < \alpha'$. Then,
\begin{center}
$\mathcal{T}_{(\alpha,\beta)}^{Ev}(X)=\mathcal{T}_{(\alpha',\beta')}(X)$ \ if and only if \ $\beta^{Ev}_1 \leq \beta' < \alpha' \leq \alpha^{Ev}_{2}$.
\end{center}
\end{theorem}

\begin{proof}
$(\Leftarrow).$ Let $\alpha',\beta' \in [0,1]$ such that $\beta^{Ev}_1 \leq \beta' < \alpha' \leq \alpha^{Ev}_2$, we need to prove that $POS_{(\alpha,\beta)}^{Ev}(X)=POS_{(\alpha',\beta')}(X)$, $NEG_{(\alpha,\beta)}^{Ev}(X)=NEG_{(\alpha',\beta')}(X)$, and $BND_{(\alpha,\beta)}^{Ev}$ $(X)=BND_{(\alpha',\beta')}(X)$.
\begin{description}
\item [$(POS_{(\alpha,\beta)}^{Ev}(X)=POS_{(\alpha',\beta')}(X))$.] Let $\bar{x} \in POS_{(\alpha,\beta)}^{Ev}(X)$. Then,  $\dfrac{|X \cap [\bar{x}]_{\mathcal{R}}|}{|[\bar{x}]_{\mathcal{R}}|} \geq \alpha^{Ev}_1$ from Definition \ref{def:alpha} (i). By hypothesis, $\alpha' \leq \alpha_2^{Ev}$. Finally, by the previous two inequalities, we obtain that $\dfrac{|X \cap [\bar{x}]_{\mathcal{R}}|}{|[\bar{x}]_{\mathcal{R}}|} \geq \alpha'$. Namely,  $\bar{x} \in POS_{(\alpha',\beta')}(X)$ from Definition \ref{def:PRS} (i). 

Let $\bar{x} \in POS_{(\alpha',\beta')}(X)$. Then, $\dfrac{|X \cap [\bar{x}]_{\mathcal{R}}|}{|[\bar{x}]_{\mathcal{R}}|} \geq \alpha'$ by Definition \ref{def:PRS} (i). Let $x_1 \in U$ such that $\beta^{Ev}_1=\dfrac{|X \cap [x_1]_{\mathcal{R}}|}{|[x_1]_{\mathcal{R}}|}$. Moreover, by hypothesis  $\beta^{Ev}_1 < \alpha'$. So, by the previous two inequalities, we get $\dfrac{|X \cap [\bar{x}]_{\mathcal{R}}|}{|[\bar{x}]_{\mathcal{R}}|} > \beta^{Ev}_1$. Then, by Definition \ref{def:alpha} (iii),  $\bar{x} \notin NEG_{(\alpha,\beta)}^{Ev}(X)$. Lastly, by \eqref{eq:tr} and by the hypothesis $BND_{(\alpha,\beta)}^{Ev}(X)=\emptyset$, we can conclude that $\bar{x} \in POS_{(\alpha,\beta)}^{Ev}(X)$. 

\item [$(NEG_{(\alpha,\beta)}^{Ev}(X)=NEG_{(\alpha',\beta')}(X))$.] Let $\bar{x} \in NEG_{(\alpha,\beta)}^{Ev}(X)$. Then, by Definition \ref{def:alpha} (iii), $\dfrac{|X \cap [\bar{x}]_{\mathcal{R}}|}{|[\bar{x}]_{\mathcal{R}}|} \leq \beta^{Ev}_1$. Additionally, we know that $\beta^{Ev}_1 \leq \beta'$ from hypothesis.   Then, $\dfrac{|X \cap [\bar{x}]_{\mathcal{R}}|}{|[\bar{x}]_{\mathcal{R}}|} \leq \beta'$, namely $\bar{x} \in NEG_{(\alpha',\beta')}(X)$ from Definition \ref{def:PRS} (ii). 

Let $\bar{x} \in  NEG_{(\alpha',\beta')}(X)$, then $\dfrac{|X \cap [\bar{x}]_{\mathcal{R}}|}{|[\bar{x}]_{\mathcal{R}}|} \leq \beta'$ from Definition \ref{def:PRS} (ii). By hypothesis, $\beta' < \alpha^{Ev}_2$. Thus, by Definition \ref{def:alpha} (ii), $\bar{x}$ cannot belong to $POS_{(\alpha,\beta)}^{Ev}(X)$. So, by \eqref{eq:tr} and the hypothesis $BND_{(\alpha,\beta)}^{Ev}(X)=\emptyset$, we can deduce that $\bar{x} \in NEG_{(\alpha,\beta)}^{Ev}(X).$ 
\item [$(BND_{(\alpha,\beta)}^{Ev}(X)=BND_{(\alpha',\beta')}(X))$.] This equality follows from $POS_{(\alpha,\beta)}^{Ev}(X)=POS_{(\alpha',\beta')}(X)$ and $NEG_{(\alpha,\beta)}^{Ev}(X)=NEG_{(\alpha',\beta')}(X)$, considering that the sets \\ $POS_{(\alpha,\beta)}^{Ev}(X), NEG_{(\alpha,\beta)}^{Ev}(X)$, and $BND_{(\alpha,\beta)}^{Ev}(X)$ (as well as $POS_{(\alpha',\beta')}(X)$, $NEG_{(\alpha',\beta')}(X)$, and $BND_{(\alpha',\beta')}(X)$) cover the universe $U$ (see \eqref{eq:tr1} and \eqref{eq:tr}). 
\end{description}

$(\Leftarrow).$  Let $\mathcal{T}_{(\alpha,\beta)}^{Ev}(X)=\mathcal{T}_{(\alpha',\beta')}(X)$, we intend to prove that $\beta^{Ev}_1 \leq \beta'$ and $\alpha' \leq \alpha^{Ev}_2$. 
\begin{description}
\item [$(\beta^{Ev}_1 \leq \beta')$.] Let $x_1 \in U$ such that $\beta^{Ev}_1 = \dfrac{|X \cap [x_1]_{\mathcal{R}}|}{|[x_1]_{\mathcal{R}}|}$. Of course, $x_1 \in NEG^{Ev}_{(\alpha,\beta)}(X)$  from Definition \ref{def:alpha} (iii).

It is clear that the inequality $\beta^{Ev}_1 > \beta'$ leads to a contradiction:
\begin{itemize}
\item if $\beta^{Ev}_1 > \beta'$, then $\bar{x} \notin NEG_{(\alpha',\beta')}(X)$ from Definition \ref{def:alpha} (iii);
\item but, this contradicts that $NEG^{Ev}_{(\alpha,\beta)}(X)=NEG_{(\alpha',\beta')}(X)$. 
\end{itemize}
So, $\beta^{Ev}_1 \leq \beta'$ must hold. 
\item [$(\alpha' \leq \alpha^{Ev}_2)$.]  Let $x_2 \in U$ such that $\alpha^{Ev}_2=\dfrac{|X \cap [x_2]_{\mathcal{R}}|}{|[x_2]_{\mathcal{R}}|}$. Then, $x_2 \in POS^{Ev}_{(\alpha,\beta)}(X)$ from Definition \ref{def:alpha} (ii). 

It is clear that the inequality $\alpha' > \alpha^{Ev}_2$ leads to a contradiction: 
\begin{itemize}
\item if $\alpha' > \alpha^{Ev}_2$, then $x_2 \notin POS_{(\alpha',\beta')}(X)$ from Definition \ref{def:alpha} (ii);
\item but, this contradicts that $ POS^{Ev}_{(\alpha,\beta)}(X)=POS_{(\alpha',\beta')}(X)$.
\end{itemize}
Finally,  $\alpha' \leq \alpha^{Ev}_2$ must hold.
\end{description}
\end{proof}

Examples of evaluative expressions satisfying the hypothesis of Theorem \ref{teo:2} can be obtained from the class   defined by \eqref{eq:delta}. Indeed, let $t \in [0,1]$, $\Delta_t$ is trivially an increasing function (i.e. $\Delta_t \in \mathcal{E}^+$) and the boundary region determined by $\Delta_t$ is always empty as shown by the following proposition. In addition, in Proposition \ref{pro:3}, the formula of the three regions that are related to $\Delta_t$ is rewritten so that the thresholds $\alpha$ and $\beta$ do not appear in it. 

\begin{proposition} \label{pro:3}
Let $t \in [0,1]$ and let $\alpha, \beta \in [0,1]$ such that $\beta < \alpha$, then
\begin{enumerate}
\item [(a)] $POS_{(\alpha,\beta)}^{\Delta_t}(X)=\left\{x \in U \ | \ \dfrac{|X \cap [x]_{\mathcal{R}}|}{|[x]_{\mathcal{R}}|} \geq t\right\}$;
\item [(b)] $NEG_{(\alpha,\beta)}^{\Delta_t}(X)=\left\{x \in U \ | \ \dfrac{|X \cap [x]_{\mathcal{R}}|}{|[x]_{\mathcal{R}}|} < t\right\}$;
\item [(c)] $BND_{(\alpha,\beta)}^{\Delta_t}(X)=\emptyset$.
\end{enumerate}
\end{proposition}
\begin{proof} (a). Let $\bar{x} \in U$. Thus, 
$\bar{x} \in POS_{(\alpha,\beta)}^{Ev}(X)$ if and only if
\begin{equation} \label{eq:pro1}
\Delta_t\left(\dfrac{|X \cap [\bar{x}]_{\mathcal{R}}|}{|[\bar{x}]_{\mathcal{R}}|}\right) \geq \alpha 
\end{equation}
from Definition \ref{def:ee} (i).

By \eqref{eq:delta}, the inequality \eqref{eq:pro1} is true if and only if $\dfrac{|X \cap [\bar{x}]_{\mathcal{R}}|}{|[\bar{x}]_{\mathcal{R}}|} \geq t$.

(b). Let $\bar{x} \in U$. Then, $\bar{x} \in NEG_{(\alpha,\beta)}^{Ev}(X)$ if and only if
\begin{equation} \label{eq:pro2}
\Delta_t\left(\dfrac{|X \cap [\bar{x}]_{\mathcal{R}}|}{|[\bar{x}]_{\mathcal{R}}|}\right) \leq \beta 
\end{equation}
from Definition \ref{def:ee} (ii).

By \eqref{eq:delta}, the inequality \eqref{eq:pro2} is true if and only if $\dfrac{|X \cap [x]_{\mathcal{R}}|}{|[x]_{\mathcal{R}}|} < t$.

(c). Notice that $\{x \in U \ | \ \dfrac{|X \cap [x]_{\mathcal{R}}|}{|[x]_{\mathcal{R}}|} \geq t\} \cup \{x \in U \ | \ \dfrac{|X \cap [x]_{\mathcal{R}}|}{|[x]_{\mathcal{R}}|} < t\}=U$. Moreover, we have proved that $POS_{(\alpha,\beta)}^{Ev}(X)=\{x \in U \ | \ \dfrac{|X \cap [x]_{\mathcal{R}}|}{|[x]_{\mathcal{R}}|} \geq t\}$ and $NEG_{(\alpha,\beta)}^{Ev}(X)=\{x \in U \ | \ \dfrac{|X \cap [x]_{\mathcal{R}}|}{|[x]_{\mathcal{R}}|} < t\}$. Hence, according to \eqref{eq:tr},  $BND_{(\alpha,\beta)}^{Ev}(X)$ must be empty.  
\end{proof}

\begin{example} \label{ex:7}
Let us focus on $\mathcal{T}^{\Delta_{0.5}}_{(\alpha,\beta)}(X)$, where $U$, $X$, and $\mathcal{R}$ are defined in Example \ref{esempio3}. By Proposition \ref{pro:3}, it is easy to verify that $POS_{(\alpha,\beta)}^{\Delta_{0.5}}(X)=C_2 \cup C_3$,  $NEG_{(\alpha,\beta)}^{\Delta_{0.5}}(X)=C_1$, and $BND_{(\alpha,\beta)}^{\Delta_{0.5}}(X)=\emptyset$ for each $\alpha, \beta \in [0,1]$ with $\beta < \alpha$. Furthermore, according to Theorem  \ref{teo:2}, $POS_{(\alpha',\beta')}(X)=C_2 \cup C_3$,   $NEG_{(\alpha',\beta')}(X)$ $=C_1$, and $BND_{(\alpha',\beta')}(X)=\emptyset$, for each $\alpha', \beta' \in [0,1]$ such that $\beta_1^{\Delta_{0.5}} \leq \beta' < \alpha' \leq \alpha_2^{\Delta_{0.5}}$, where $\beta_1^{\Delta_{0.5}}=0$, $\alpha_2^{\Delta_{0.5}}=0.9$. For example, if we choose $\alpha'=0.2$ and $\beta'=0.7$, we obtain $POS_{(0.7,0.2)}(X)=C_2 \cup C_3$,   $NEG_{(0.7,0.2)}(X)=C_1$, and $BND_{(0.7,0.2)}(X)=\emptyset$. 
\end{example}

Now, let us suppose that the negative region is empty.

\begin{theorem} \label{teo:3}
Let $Ev \in \mathcal{E}^+$ such that $NEG_{(\alpha,\beta)}^{Ev} =\emptyset$ and $POS_{(\alpha,\beta)}^{Ev}, BND_{(\alpha,\beta)}^{Ev} \neq \emptyset$. Let $\alpha', \beta' \in [0,1]$ such that $\beta' < \alpha'$. Then,
\begin{center}
$\mathcal{T}_{(\alpha,\beta)}^{Ev}(X)=\mathcal{T}_{(\alpha',\beta')}(X)$ \ if and only if \
   $\beta' \in [0, \beta^{Ev}_2)$ and $\alpha' \in (\alpha^{Ev}_1, \alpha^{Ev}_2]$.
\end{center}
\end{theorem}
\begin{proof}
The proof is similar to that of Theorems \ref{teo:connection} and \ref{teo:2}. So, it is omitted. 
\end{proof}
\begin{example} \label{ex:8}
Let $U=\{u_1, \ldots, u_{30}\}$. We supposed that $U$ is divided into the following  equivalence classes: $C_1=\{u_1, \ldots, u_5\}$, $C_2=\{u_6, \ldots, u_{10}\}$, and $C_3=\{u_{11}, \ldots, u_{30}\}$. 

Also, let $X=\{u_1, \ldots u_{28}\}$, we are interested in $\mathcal{T}^{BiVe}_{(0.8,0.4)}(X)$. Then, \begin{equation*} 
\dfrac{|X \cap [x]_\mathcal{R}|}{|[x]_\mathcal{R}|}=\begin{cases}
0.9 & \mbox{ if } x \in C_3,\\
1 & \mbox{ if } x \in C_1 \cup C_2.
\end{cases}
\end{equation*}
and 
 \begin{equation*}
 Bi Ve \left(\frac{|X \cap [x]_\mathcal{R}|}{|[x]_\mathcal{R}|}\right)= \begin{cases} 0.58 & \mbox{ if } x \in C_3 , \\  1 & \mbox{ if } x \in C_1 \cup C_2. 
 \end{cases}
\end{equation*}
Thus, by Definition \ref{def:ee},  $POS_{(0.8,0.4)}^{BiVe}(X)=C_1 \cup C_2$, $BND_{(0.8,0.4)}^{BiVe}(X)=C_3$, and $NEG_{(0.8,0.4)}^{BiVe}(X)=\emptyset$. 
Moreover, $\beta^{Ev}_2=0.9$, $\alpha_1^{Bi Ve}=0.9$, and  $\alpha_2^{Bi Ve}=1$. So, according to Theorem \ref{teo:3}, $POS_{(\alpha',\beta')}(X)=C_1 \cup C_2$, $BND_{(\alpha',\beta')}(X)=C_3$, and $NEG_{(\alpha',\beta')}(X)=C_3$ for each $\beta' \in [0, 0.9)$ and $\alpha' \in (0.9, 1]$. For example, if $\alpha'=0.95$ and $\beta'=0.8$, then $POS_{(0.8,0.95)}(X)=C_1 \cup C_2$, $BND_{(0.8,0.95)}(X)=C_3$, and $NEG_{(0.8,0.95)}(X)=\emptyset$.  
\end{example}
Finally, the case of an empty positive region.
\begin{theorem} \label{teo:4}
Let $Ev \in \mathcal{E}^+$ such that $POS_{(\alpha,\beta)}^{Ev} =\emptyset$ and $NEG_{(\alpha,\beta)}^{Ev}, BND_{(\alpha,\beta)}^{Ev} \neq \emptyset$. Let $\alpha', \beta' \in [0,1]$ such that $\beta' < \alpha'$. Then,
\begin{center}
$\mathcal{T}_{(\alpha,\beta)}^{Ev}(X)=\mathcal{T}_{(\alpha',\beta')}(X)$ \ if and only if \
   $\beta' \in [\beta^{Ev}_1, \beta^{Ev}_2)$ and $\alpha' \in (\alpha^{Ev}_1, 1]$.
\end{center}
\end{theorem}
\begin{proof}
The proof is similar to that of Theorems \ref{teo:connection} and \ref{teo:2}. So, it is omitted. 
\end{proof}
\begin{example}
Consider the universe $U$ and the equivalence classes $C_1, C_2$, and $C_3$, which are defined by Example \ref{ex:8}. Let $X=\{u_1, u_6, u_{11}, \ldots u_{28}\}$, we focus on $\mathcal{T}_{(0.7,0.2)}^{BiVe}(X)$. Then, \begin{equation*} 
\dfrac{|X \cap [x]_\mathcal{R}|}{|[x]_\mathcal{R}|}=\begin{cases}
0.5 & \mbox{ if } x \in C_1 \cup C_2,\\
0.9 & \mbox{ if } x \in C_3.
\end{cases}
\end{equation*}
and
\begin{equation*} 
 Bi Ve \left(\frac{|X \cap [x]_\mathcal{R}|}{|[x]_\mathcal{R}|}\right)= \begin{cases} 0.58 & \mbox{ if } x \in C_3 , \\  0 & \mbox{ if } x \in C_1 \cup C_2. 
 \end{cases}
 \end{equation*}
 By Definition \ref{def:ee},  $NEG_{(0.7,0.2)}^{Bi Ve}(X)=C_1 \cup C_2$,  $BND_{(0.7,0.2)}^{Bi Ve}(X)=C_3$, and \\$POS_{(0.7,0.2)}^{Bi Ve}(X)=\emptyset$. Also, $\beta^{Ev}_1=0.5$, $\beta^{Ev}_2=0.9$, $\alpha_1^{Ev}=0.9$.  Thus, according to Theorem \ref{teo:4}, $NEG_{(\alpha',\beta')}(X)=C_1 \cup C_2$,  $BND_{(\alpha',\beta')}(X)=C_3$, and $POS_{(\alpha',\beta')}(X)=\emptyset$ for each $\beta' \in [0.5,0.9)$ and $\alpha' \in (0.9,1]$. For example, let $(\alpha',\beta')=(0.95, 0.6)$, we can easily verify that $NEG_{(0.95,0.6)}(X)=C_1 \cup C_2$,  $BND_{(0.95,0.6)}(X)=C_3$, and $POS_{(0.95,0.6)}(X)=\emptyset$. 
\end{example}
\section{Conclusions and future directions}

This work proposes a novel model for three-way decisions based on the concept of evaluative linguistic expressions. Thus, a new way is provided to divide the initial universe into three regions with the corresponding decisions rules. Moreover, our results allow decision-makers to give a linguistic  interpretation to the regions already obtained using the probabilistic approach. Let us indicate some  possible directions to continue this work. Firstly, we need to extend the results of Section \ref{sec:4} to the evaluative expressions that are not necessarily represented by increasing functions. Then, we want to deepen the study of linguistic-regions by comparing our methods with those presented in \cite{yao2012outline}. In addition, we intend to understand how the decisions about the elements change using different evaluative expressions. Finally, we could analyze the logical relations between the linguistic regions determined by a given evaluative expression and investigate their consequences in terms of decisions by constructing an hexagon of opposition. 
\bibliographystyle{unsrt}

\bibliography{bibliografiaEx}
\end{document}